\documentclass[letterpaper, 10 pt, conference]{ieeeconf}

\IEEEoverridecommandlockouts                              
\usepackage[utf8]{inputenc}
\usepackage{graphicx}   
\usepackage{hyperref}
\usepackage{xcolor}
\usepackage{amsmath,amssymb, amsfonts, mathtools}
\mathtoolsset{showonlyrefs}
\usepackage[utf8]{inputenc}
\usepackage[english]{babel}
\usepackage[caption=false,font=footnotesize]{subfig}
\usepackage{mathrsfs}
\usepackage{algorithm}
\usepackage{algpseudocode}

\usepackage{amsmath,amsthm,amssymb}
\usepackage{mathtools}
\mathtoolsset{showonlyrefs}
\usepackage{mathrsfs}
\usepackage{comment}
\newtheorem{theorem}{Theorem}

\newtheorem{proposition}{Proposition}

\DeclareMathOperator{\co}{co}

\newcommand{\R}{\mathbb{R}}

\newcommand{\tr}{^T}

\newcommand{\mc}{\mathcal}

\definecolor{blue}{rgb}{0.0, 0.0, 1.0}
\DeclareMathOperator*{\minimize}{minimize}
\DeclareMathOperator*{\subjto}{subject\,to}

\title{\LARGE \bf
Data-Driven Adaptive Task Allocation for Heterogeneous Multi-Robot Teams Using Robust Control Barrier Functions}
\author{Yousef Emam$^{1}$, Gennaro Notomista, Paul Glotfelter$^{2}$,  Magnus Egerstedt$^{1}$
	\thanks{*This work was supported by the Army Research Lab through ARL DCIST CRA W911NF-17-2-0181.}
	\thanks{$^{1}$Y. Emam, M. Egerstedt are with the Institute for Robotics and Intelligent Machines, Georgia Institute of Technology, Atlanta, GA 30332, USA {\tt\small\{emamy, magnus\}@gatech.edu}}
	\thanks{$^{2}$P. Glotfelter is with Optimus Ride, MA 0$7$10, Massachusetts, USA {\tt\small pglotfel@gmail.com}}
}
\date{\today}

\begin{document}

\maketitle

\begin{abstract}
Multi-robot task allocation is a ubiquitous problem in robotics due to its applicability in a variety of scenarios. Adaptive task-allocation algorithms account for unknown disturbances and unpredicted phenomena in the environment where robots are deployed to execute tasks. However, this adaptivity typically comes at the cost of requiring precise knowledge of robot models in order to evaluate the allocation effectiveness and to adjust the task assignment online. As such, environmental disturbances can significantly degrade the accuracy of the models which in turn negatively affects the quality of the task allocation. In this paper, we leverage Gaussian processes, differential inclusions, and robust control barrier functions to learn environmental disturbances in order to guarantee robust task execution.
We show the implementation and the effectiveness of the proposed framework on a real multi-robot system.
\end{abstract}

\section{Introduction}
\label{sec:intro}

Multi-Robot Task Allocation (MRTA) is typically required in a number of applications such as search and rescue and precision agriculture \cite{taxonomy}. Many of those applications involve a variety of concurrent tasks, which in turn requires the robots to have varying capabilities. One such example is disaster relief, where the team of robots must both seek and rescue the victims \cite{hussein2014multi}. In such scenarios, where the team is heterogeneous, it is crucial to consider the specialization of each robot with respect to the different available tasks \cite{parker1994heterogeneous} to perform an effective task allocation.

There exists a large body of work on heterogeneous MRTA such as \cite{parker1994heterogeneous, palmieri2018self, 4803959}. In a broad sense, heterogeneous MRTA can be seen as an extension of MRTA with the additional complexity that assigning the same task to different robots may result in different costs \cite{taxonomy}. Therefore, the underlying approaches are similar to homogeneous MRTA and can be classified as market-based, centralized, or decentralized approaches as highlighted in \cite{khamis2015multi}. Moreover, to differentiate between the various capabilities of the different robots, a common trend in existing frameworks is to assign each agent a specialization towards each task (e.g., \cite{higuera2013fair, zlot2006market, iocchi2003distributed, iijima2017adaptive}). That is, the discrepancy in the utility gained from assigning a task to one robot versus another is encoded via these specializations. Consequently, using inaccurate specialization values when planning can degrade the quality of allocations, and in the worst case, can assign robots to tasks which they are not capable of accomplishing. This possibility motivates the need for a framework that learns these specializations on-the-fly in the case they are unknown a priori.


Toward this goal, in \cite{emam2020adaptive}, the authors propose an adaptive task allocation and execution scheme based on \cite{notomista2019optimal} that dynamically updates these specializations. As opposed to the conventional approach of sequentially allocating tasks then executing them, the framework in \cite{notomista2019optimal} simultaneously allocates the tasks and computes the control input for the robots to execute them. Consequently, this approach has the benefit of explicitly accounting for the control input in the cost function when allocating tasks. Moreover, as presented and discussed in \cite{emam2020adaptive}, the task execution is performed as the minimization of energy-like functions, which by nature encode a measure of progress and, as such, are readily used in specialization adaptation. The specialization is learned by comparing the expected versus actual progress made by the robots towards minimizing the energy functions. For example, if the robot performed worse than expected, then its specialization toward that task is decreased. In \cite{emam2020adaptive}, the authors consider tasks that can be encoded as positive definite functions of the state. We note that although not all tasks can be encoded as energy functions, this formulation is applicable to a wide variety of coordinated control tasks such as formation control and coverage \cite{cortes2017coordinated}. 

The approach in \cite{emam2020adaptive}, however, requires the knowledge of a predefined, precise dynamical model of the robots to evaluate the expected task progresses. These models do not account for environmental disturbances or unknown phenomena; therefore, the quality of both the allocation and the execution of tasks by the robots may be affected by these disturbances. More importantly, even small environmental disturbances may result in the deterioration of the estimated specialization of the robots. In other words, the framework cannot distinguish between disturbances that the robots can and cannot overcome. This negatively affects the ability to allocate tasks based on specialization in an effective fashion.

Motivated by these limitations, in this paper, we propose a novel framework that leverages Gaussian Processes (GPs) and Robust Control Barrier Functions (RCBFs) as introduced in \cite{emam2019robust} to learn and model the disturbed system as well as to ensure the execution of tasks under these disturbances. We also introduce a new update law for the specializations based on the learned disturbed dynamic models. This allows the framework to distinguish between disturbances that the robots can and cannot overcome---the former being due to model errors, while the latter is caused by the actual incapability of the robots at performing the task.


\section{Background Material} 
\label{sec:background}

In this section, we introduce CBFs and RCBFs for disturbed dynamical systems, which are leveraged in the task allocation component of the proposed MRTA framework.

\subsection{Control Barrier Functions} 
\label{subsec:control-barrier-functions}
Control Barrier Functions (CBFs) are formulated with respect to control systems \cite{ames2014,xu2015,AmesBarriers,ogren2006autonomous}, and this work considers control-affine systems
\begin{equation}
    \label{eq:control-affine}
    \dot{x}(t) = f(x(t)) + g(x(t))u(x(t)) , x(0) = x_{0} ,
\end{equation}
where $f : \mathbb{R}^{n} \to \mathbb{R}^{n}$, $g : \mathbb{R}^{n} \to \mathbb{R}^{n \times m}$, and $u : \mathbb{R}^{n} \to \mathbb{R}^{m}$ are continuous. A set $\mathcal{C}$ is called forward invariant with respect to \eqref{eq:control-affine} if given a solution (potentially nonunique) to \eqref{eq:control-affine} $x : [0, t_{1}] \to \mathbb{R}^{n}$,
$x_{0} \in \mathcal{C} \implies x(t) \in \mathcal{C}, \forall t \in [0, t_{1}]$.

Barrier functions guarantee forward invariance of a particular set that typically represents a constraint in a robotic system, such as collision avoidance or connectivity maintenance.   Specifically, a barrier function is a continuously differentiable function $h : \mathbb{R}^{n} \to \mathbb{R}$ (sometimes referred to as a candidate barrier function), and the so-called safe set $\mathcal{C} \subset \mathbb{R}^{n}$ is defined as the super-zero level set of $h$
$\mathcal{C} = \{x \in \mathbb{R} : h(x) \geq 0\}$.
Now, the goal becomes to ensure the forward set invariance of $\mathcal{C}$, which can be done equivalently by guaranteeing positivity of $h$ along trajectories to \eqref{eq:control-affine}.

Positivity can be shown if there exists a locally Lipschitz extended class-$\mathcal{K}$ function $\gamma : \mathbb{R} \to \mathbb{R}$ and a continuous function $u : \mathbb{R}^{n} \to \mathbb{R}^{m}$ such that
\begin{equation}
    \label{eq:barrier-certificate-reg}
     L_{f}h(x) + L_{g}h(x)u(x) \geq -\gamma(h(x)), \forall x \in \mathbb{R}^{n},
\end{equation}
where $L_{f}h(x)=\nabla h(x)^{\top}f(x)$ and $L_{g}h(x)=\nabla h(x)^{\top}g(x)$ denote the Lie derivatives of $h$ in the directions $f$ and $g$ respectively.  A function is class-$\mathcal{K}$ if it is continuous, strictly increasing, and $\gamma(0) = 0$. If the above conditions hold, then $h$ is called a valid CBF for \eqref{eq:control-affine} \cite{ames2019control}.

\subsection{Robust CBFs}
\label{subsec:robust-CBFs}

As introduced in \cite{emam2019robust}, RCBFs guarantee the safety of disturbed dynamical systems obeying the following differential inclusion
\begin{equation} 
    \label{eq:control-affine-disturbed}
    \dot{x}(t) \in f(x(t)) + g(x(t))u(x(t)) + D(x(t)), x(0) = x_{0} ,
\end{equation}
where $f$, $g$, $u$ are as in \eqref{eq:control-affine} and $D : \mathbb{R}^{n} \to 2^{\mathbb{R}^{n}}$ (the disturbance) is an upper semi-continuous set-valued map that takes nonempty, convex, and compact values. Note that $2^{\mathbb{R}^{n}}$ refers to the power set of $\mathbb{R}^{n}$, and that the assumptions made on the disturbance are conditions to guarantee the existence of solutions \cite{cortes2008discontinuous}.

We refer the reader to  \cite{glotfelter2017nonsmooth} for more details on the use of differential inclusions with barrier functions. Moreover, it was shown that for a specific form of $D$, we can recover a similar formulation of regular CBFs as in \eqref{eq:barrier-certificate-reg}
 with almost no additional computational cost as stated in the following theorem.  An important aspect of this theorem is that forward invariance is guaranteed for all trajectories of \eqref{eq:control-affine-disturbed}.
\begin{theorem}{\cite{emam2019robust}}
    \label{mainTheorem}
    Let $h : \mathbb{R}^{n} \to \mathbb{R}$ be a continuously differentiable function.  Let $\psi_{i} : \mathbb{R}^{n} \to \mathbb{R}^{n}$, $i \in \{1, \hdots, p\}$ be a set of $p > 0$ continuous functions, and define the disturbance $D : \mathbb{R}^{n} \to 2^{\mathbb{R}^{n}}$ as 
    \begin{equation}
        D(x') = \co \Psi(x') = \co\{\psi_1(x')\ldots\psi_p(x')\}, \forall x' \in \mathbb{R}^{n} .
    \end{equation}
    If there exists a continuous function $u : \mathbb{R}^{n} \to \mathbb{R}^{m}$ and a locally Lipschitz extended class-$\mathcal{K}$ function $\gamma : \mathbb{R} \to \mathbb{R}$ such that 
    \begin{equation} 
        \label{eq:mainTheorem}
            \begin{split}
                &L_{f}h(x') + L_{g}h(x')u(x') \geq \\
                & -\gamma(h(x')) - \min \nabla h(x')^{\top} \Psi(x'), \forall x' \in \mathbb{R}^{n} ,
            \end{split}
    \end{equation}
    then $h$ is a valid RCBF for \eqref{eq:control-affine-disturbed}.
\end{theorem}

Note that similarly to \cite{xu2015robustness}, in this paper, we prove the asymptotic stability of the safe set of RCBFs. Moreover, since the main contribution of this paper is a MRTA framework that accounts for environmental disturbances using differential inclusions and RCBFs, in the next section, we present the MRTA framework from \cite{emam2020adaptive} which we build upon. 

\section{Adaptive Task Allocation and Execution for Heterogeneous Robot Teams}
\label{sec:adaptiveTA}
This section briefly presents the adaptive task allocation and execution framework for heterogeneous multi-robot teams introduced in \cite{emam2020adaptive, notomista2019optimal} and discusses its shortcomings. For the remainder of this section, we consider a multi-robot team consisting of $N$ robots which are to execute $M$ tasks, denoted by $T_j$ for $j\in \{1,\ldots,M\} \overset{\Delta}{=} \mc M$. For the sake of generality, we assume that each robot $i$, where $i\in \{1,\ldots,N\} \overset{\Delta}{=} \mc N$, in the multi-robot system can be modeled as the following control-affine dynamical system 
$\dot x_i = f(x_i) + g(x_i) u_i$,
where $x_i \in \mathbb{R}^n$ is the state, $u_i  \in \mathbb{R}^m$ is the input, and $f$ and $g$ are as in \eqref{eq:control-affine}. Note that here, and for the remainder of the section, we omit the dependence of the different variables on time for brevity.

\subsection{Single-Robot Multi-Task Execution}
\label{subsec:SR-MT}
 In \cite{notomista2018constraint}, a formulation for the execution of tasks was developed, which encoded the completion of each task $j$ by robot $i$ as the safe-set of a CBF $h_{ij}$. Thus, the control signal $u_i(t)$ needed to execute the task can be obtained by solving the following optimization problem at all time
\begin{equation}
\label{eqn:const_opt}
\begin{aligned} 
\minimize_{u_i,\delta_{ij}}~~&\|u_i\|^2 + \delta_{ij}^2 \\
\subjto~~&L_f h_{ij}(x_i) + L_g h_{ij}(x_i) u_i \geq -\gamma(h_{ij}(x_i)) - \delta_{ij},
\end{aligned}
\end{equation}
where $\delta_{ij}$ is a slack variable which represents the extent to which the constraint corresponding to the task execution of task $T_j$ can be violated and $\gamma$ is an extended class-$\mathcal{K}$ function \cite{cbftutorial}. Note here that the control input $u_i$ is computed with respect to a pre-defined model of the dynamics. As such, in the case where environmental disturbances are present, the resulting control input may yield poor performance. This constitutes the first motivation for learning a real-time model of the disturbances. 

To extend the framework above to the multi-task case, the authors in \cite{notomista2018constraint} introduce constraints on the slack variables $\delta_{ij}$ to allow the robots to prioritize some tasks over others. To achieve this, they introduce the variable $\alpha_i = [\alpha_{i1},\ldots,\alpha_{iM}]^T \in \{0,1\}^M$, whose entries indicate the priorities of the tasks for robot $i$ (i.e. $\alpha_{im} = 1$ $\Longleftrightarrow$ task $T_m$ has the highest priority for robot $i$). Therefore, the prioritization is achieved by ensuring that the following implication holds
\begin{equation}\label{eq:alphadelta}
\alpha_{im} = 1 \quad\Rightarrow\quad \delta_{im} \leq \frac{1}{\kappa}\delta_{in}\quad \forall n\in \mathcal{M},~n\neq m,
\end{equation}
where $\kappa > 1$ allows us to encode how the task priorities impact the relative effectiveness with which robots perform different tasks.

\subsection{Multi-Robot Multi-Task Allocation and Execution}

Finally, to extend the single-robot multi-task execution to multiple robots, the authors in \cite{notomista2019optimal} combine the optimization problems \eqref{eqn:const_opt} for each robot and add a task allocation term in the cost that allows the team to reach a desired global specification, resulting in the following optimization problem
\begin{subequations} \label{eq:allocationalgorithm}
\vspace{-5mm}
\begin{align}
\minimize_{u,\delta,\alpha} ~& C\|\pi^* - \pi_h(\alpha)\|_{T}^2 + \sum_{i = 1}^{N} \Big( \|u_i\|^2 + l \|\delta_i \|_{S_i}^2 \Big) \label{eq:miqp:a}\\ 
\subjto~~&L_f h_{m}(x) + L_g h_{m}(x) u_i \geq -\gamma(h_{m}(x)) - \delta_{im} \label{eq:miqp:b} \\
&\mc P\delta_i \le \Omega(\alpha_i) \label{eq:miqp:c}\\
&|u_i| \leq u_\text{max} \label{eq:miqp:d}\\
&\hspace{3cm}\forall i \in \mc N,~\forall n,m \in \mc M.
\end{align}
\noeqref{eq:miqp:a}\noeqref{eq:miqp:b}\noeqref{eq:miqp:c}
\end{subequations}
\hspace{-0.2cm}where $x_i$ and $u_i$ denote the state and input of robot $i$ respectively, and $x = [ x_1 \tr x_2 \tr \ldots x_N\tr]$ denotes the ensemble state. The term $\|\pi^* - \pi_h(\alpha)\|_{T}^2$ serves to steer the distribution of the robots over the tasks, encoded by  $\pi_h \colon \mathbb{R}^{N \times M} \rightarrow \mathbb{R}^{M}$, towards a desired distribution $\pi^*$.  In \eqref{eq:miqp:a}, $C$ and $l$ are scaling constants allowing for a trade-off between the global specifications term  and allowing individual robots to expend the least amount of energy possible. The term $\|\delta_i \|_{S_i}^2$ accounts for the heterogeneity of the robot team by not penalizing robots' slack variables corresponding to tasks for which they are not suitable. This is achieved by leveraging the specialization $s_{ij} \in [0,1]$ of robot $i$ towards task $j$ as described in \cite{emam2020adaptive}. The constraint \eqref{eq:miqp:c}, where $\mc P \in \mathbb{R}^{q \times m}$, $\Omega \colon \mathbb{R}^m \rightarrow \mathbb{R}^{q}$  and $q$ denotes the number of desired constraints, encodes the relation described in \eqref{eq:alphadelta}. Finally, \eqref{eq:miqp:d} encodes the control limits of the robots. We refer the reader to \cite{notomista2019optimal} for more details. It is important to note that, since the allocation and execution of tasks are intertwined, environmental disturbances can not only inhibit the ability of a robot to accomplish a task but can also negatively affect its allocation. 

\subsection{Adaptive Specialization for Dynamic Environments}
\label{subsec:adaptiveTA}

The specialization parameters $s_{ij}$ which encode the effectiveness of robot $i$ at performing task $j$, are updated at fixed intervals $dt$ through the following update law
\begin{equation}
    \label{eq:spUpdate1}
    s_{ij}[k + 1]= \text{min}(\text{max}(s_{ij}[k] + \beta_{1} \alpha^*_{ij}[k]  \Delta h_{ij}[k],0),  1),
\end{equation}
where $\beta_{1} \in \R_{>0}$ is a constant controlling the update rate and $\alpha^*_{ij}[k]$ is obtained from the solution of the optimization program \eqref{eq:allocationalgorithm} at time step $k$. Note that the update only occurs for tasks to which the robots are assigned since $\alpha_{ij}[k] = 1$ if and only if robot $i$ is assigned to task $j$ at time step $k$. Lastly, the difference between the modeled and actual progresses is given by 
$\Delta h_{ij}[k] = h_{ij}(x_i^{\text{act}}[k]) - h _{ij}(x_i^{\text{sim}}[k])$,
where $h_{ij}(x_i^{\text{sim}}[k])$ and $h_{ij}(x_i^{\text{act}}[k])$ are the simulated and actual cost function values of agent $i$ for task $j$ at step $k$. The simulated state is obtained as follows
\begin{align}
    \label{eq:x_sim}
    x_i^{\text{sim}}[k] &= x_i^{\text{act}}[k-1] \\  &+  (f(x_i^{\text{act}}[k-1]) + g(x_i^{\text{act}}[k-1])u_i^{\ast}[k-1])dt ,
\end{align}
where $x_i^{\text{act}}[k]$ denote the actual state of robot $i$ and $x_i^{\text{sim}}[k]$ denote the simulated state which assumes that the robot obeyed its nominal dynamics. Here, again, the update law relies on the simulated progress which in turns leverages the pre-defined dynamics models. Since the update law is applied at each time step, any disturbance will thus cause drastic changes in the specializations over long duration. As such, the framework is incapable of distinguishing between disturbances with different severeties. 

Overall, as discussed in this section, the adaptive task allocation and execution framework relies heavily on a pre-defined dynamics model of the robots. This model is used for task execution, affects the allocation of tasks amongst the robots, and is also leveraged in the adaptive specialization update law. Consequently, it is clear that disturbances causing non-negligible changes in the dynamics can severely affect the performance of the framework. As such, in the next section, we propose a new version of the framework, that leverages GPs, differential inclusions, and RCBFs to learn and incorporate the disturbances.

\section{Data-driven adaptive MRTA}
\label{sec:adaptiveTA-robustCBF}

This section contains the main contribution of the paper: a data-driven adaptive task allocation and execution framework that explicitly accounts for the environmental disturbances in the dynamics models. As mentioned in the previous section, in the original framework proposed in \cite{emam2020adaptive}, unmodeled disturbances affecting the robots can degrade the quality of the allocation and execution of tasks as well as the update of the specializations. This drawback can be alleviated by explicitly accounting for the disturbances in the robot dynamics model, and in turn using the new model for task execution and updating the specializations, which is what we accomplish in this section.


\subsection{Modelling and Learning the disturbance}
\label{subsec:learning_disturbance}
In this subsection, we demonstrate how we can leverage GPs and differential inclusions to learn the environmental disturbances and model the disturbed dynamics.
We assume that each robot can be modelled by the differential inclusion as in \eqref{eq:control-affine-disturbed}: 
    $\dot{x}_i(t) \in f(x_i(t)) + g(x_i(t))u_i(x_i(t)) + D_i(x_i(t))$, 
where the disturbance set $D_i$ is a convex hull of $p$ continuous functions as in Theorem~\ref{mainTheorem}
\begin{equation}
    D_i(x) = \co \Psi_i(x) = \co\{\psi_{i1}(x)\ldots\psi_{ip}(x)\}, \forall x \in \mathbb{R}^{n}.
\end{equation}
Note that the choice of modelling the disturbance as additive (i.e. does not multiply the control input) is not necessary. A multiplicative disturbance can be used with only minor changes to the proposed approach.

Moreover, as mentioned above, the environmental disturbances acting on the robots may be unknown apriori. Therefore, it is  necessary to learn $D_i \;  \forall i$. Towards this goal, we propose learning the disturbance using GPs similarly to \cite{wang2018safe}. This is achieved by collecting data points to construct a dataset for each robot $\mathcal{D}_i = \{x_i^{(k)}, y_i^{(k)}\}^{n_d}_{k=1}$ where the labels $y_i^{(k)}$ are given by
\begin{equation}
    y_i^{(k)} = \hat{\dot{x}}_i^{(k)} - f(x_i^{(k)}) + g(x_i^{(k)})u_i^{(k)}, \; \forall k,
\end{equation}
and $\hat{\dot{x}}_i^{(k)}$ is the measured velocity of the state of robot $i$ for data-point $k$. Note that $k$ here is the index of a data-point in the dataset, and does not encode any relation to time. Since $y_i^{(k)} \in \mathbb{R}^n$, we train one GP per dimension for a total of $n$ GP models, and the disturbance estimate for a query point $x$ is obtained as
\begin{equation}
\label{eq:disturb_estimate}
[\bar{D}_i(x)]_d = \mu_{i,d}(x) + [-k_c \sigma_{i,d}(x), k_c \sigma_{i,d}(x)],
\end{equation}
where $\mu_{i,d}(x)$ and $\sigma_{i,d}(x)$ are robot $i$'s $d$th dimension mean and standard-deviation predictions for query point $x$ and $k_c$ is a confidence parameter (e.g. $k_c = 2$ indicates a confidence of $95.45\%$). Note that the disturbance estimate $\bar{D}_i(x)$ is a vector where each entry is a convex-hull. To recover the representation needed for Theorem~\ref{mainTheorem}, we define $\Psi_i(x)$ as the $2^n$ extrema of $\bar{D}_i(x)$.

\subsection{Execution of Tasks under Environmental Disturbances}
\label{subsec:task_exec_robust}
 Given that the disturbed dynamics are modelled using differential inclusions, regular CBFs can no longer be utilized for task execution as in the original framework. Therefore, we leverage RCBFs to ensure the execution of the tasks with respect to the disturbed dynamics. Note, however, that Theorem~\ref{mainTheorem} states conditions for a function to be a valid RCBF (i.e., to render its super-zero level set forward invariant), but does not claim  stability of the safe-set. Therefore, since this property is critical to the execution of tasks using CBFs, in order to guarantee the execution of tasks, we first prove that the super-zero level set of RCBFs is asymptotically stable similarly to the work in \cite{xu2015robustness}. More specifically, since differential inclusions consider sets of trajectories, we prove that all possible trajectories are asymptotically stable.

\begin{proposition}
\label{prop:robustness}
    Let the function $h$ and the disturbance $D$ be as in Theorem~\ref{mainTheorem}. If there exists a continuous function $u : \mathbb{R}^{n} \to \mathbb{R}^{m}$ and a locally Lipschitz extended class-$\mathcal{K}$ function $\gamma : \mathbb{R} \to \mathbb{R}$ such that 
    \begin{equation} 
        \label{eq:robustnessIneq}
            \begin{split}
                & L_{f}h(x') + L_{g}h(x')u(x') \geq \\
                & -\gamma(h(x')) - \min \nabla h(x')^{\top} \Psi(x'), \forall x' \in \mathbb{R}^{n} ,
            \end{split}
    \end{equation}
    then $\mathcal{C}$ is asymptotically stable.
\end{proposition}

\begin{proof}
Define the Lyapunov function 
\begin{align}
    V(x) &= 0, && \forall x\in \mathcal{C},\\
         &= -h(w) && o.w..
\end{align}
Then, from \eqref{eq:robustnessIneq}, we obtain that $\forall x \in \mathbb{R}^n \setminus \mathcal{C}$ 
    \begin{equation} 
            \begin{split}
                & \nabla h(x)^{\top}(f(x) + g(x)u(x)) + \min \nabla h(x)^{\top} \Psi(x)\geq \\
                & -\gamma(h(x)),
            \end{split}
    \end{equation}
which is equivalent to
\begin{align}
    &\nabla h(x)^{\top}\dot{x} \geq - \gamma(h(x)), \\
    & \; \; \; \forall \dot{x} \in f(x) + g(x)u(x) + D(x),
\end{align}
hence $\dot{V}(x) \leq \gamma(-V(x))$, since $V(x) > 0 \; \forall x \in \mathbb{R}^n \setminus \mathcal{C}$, then we get $\dot{V}(x) < 0$. Combining the latter with the forward invariance of $\mathcal{C}$ from Theorem~\ref{mainTheorem} and  the result of Theorem 2.8 from \cite{lin1996smooth} we obtain that $\mathcal{C}$ is globally asymptotically stable for all trajectories. Note that although Theorem 2.8 from \cite{lin1996smooth} requires the function $V$ to be smooth, the function can be smoothed using Proposition 4.2 in \cite{lin1996smooth}.
\end{proof}
The result in Proposition~\ref{prop:robustness} allows us to encode the execution of tasks as a constraint, similarly to the undisturbed system as
\begin{align}
\label{eq:miqp-b-robust}
L_f h_{ij}(x_i) +& L_g h_{ij}(x_i) u_i \geq  \\ &-\gamma(h_{ij}(x_i)) - \min \nabla h_{ij}(x_i)^{\top} \Psi(x_i).
\end{align}
Note that accounting for the disturbance shrinks the size of the set of control inputs that can satisfy the inequality. Moreover, in cases where the disturbance is large, it may be the case that the robots cannot satisfy the inequality while respecting their control limits. This indicates that the disturbance is too large for the robot to handle and as such the robot is incapable of executing the task. To encode this, we limit the magnitude of the learned disturbance. Specifically, we enforce $||\psi_{ij}(x)||_{\infty} \leq d_{\text{max}}, \; \forall j \in \{1,\ldots, p\}$, where $d_{\text{max}}$ is a user-defined threshold. The latter allows the robots to only execute tasks for which they can overcome the disturbance, while also accounting for disturbances that the robots cannot overcome through the adaptive specialization law  as discussed in the next subsection.
Regarding the incurred computational complexity, since the term $\min \nabla h_{m}(x)^{\top} \Psi(x)$ does not depend on the control input $u$, the additional cost is $O(Np)$, where $N$ is the number of robots, and $p$ is the number of points forming the convex hull of the disturbance.

\subsection{Adaptive Specialization using Differential Inclusions}
\begin{figure}[t]
    \centering
    \includegraphics[width=0.98\linewidth]{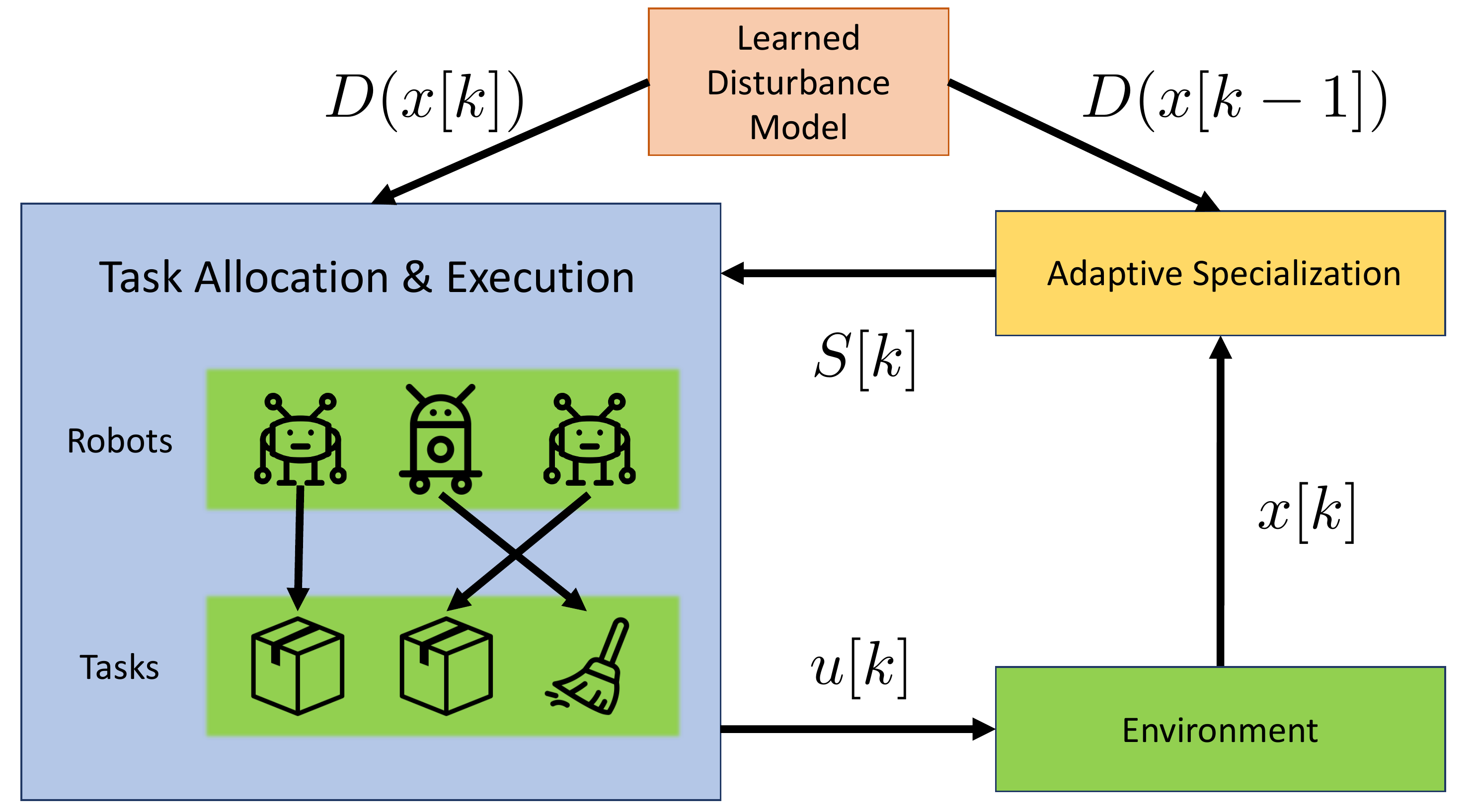}
    \caption{High-level diagram of the proposed framework. The main optimization program solved is denoted by the blue box, where the task execution is encoded using the RCBFs obtained from the data-driven models (in our case GPs), and specializations from the adaptive specialization update. The specialization update is computed using the disturbance from the last step to approximate the nominal progress, along with the current state to estimate the actual progress.}
    \label{fig:high-level-diagram}
\end{figure}

The adaptive specialization law \eqref{eq:spUpdate1} depends on a model of the robots' dynamics that does not account for environmental disturbances. Hence, as reflected in \eqref{eq:x_sim}, if the robot consistently performs worse than expected with respect to achieving a certain task, its specialization towards that task would diminish to zero. In other words, the update law is unable to differentiate between environmental disturbances that the robot can and cannot overcome. Therefore, to account for the environmental disturbances when computing the modelled progress, we propose to update the specializations by leveraging the differential inclusions learned using the procedure described in Section~\ref{subsec:learning_disturbance}. 

Since each robot is modelled through a differential inclusion, there are several possible trajectories which the robot can take, allowing for numerous modelling options. For example, the modelled progress in \cite{emam2020adaptive} is with respect to the undisturbed trajectory; however, another valid choice would be to compute the trajectory with respect to the expected disturbance. We make the conservative choice of focusing on the worst-case trajectory (i.e. the trajectory that results in the least progress), yielding an update law given by 
\begin{equation} \label{eq:spUpdate-robust}
    s_{ij}[k + 1]= \text{min}(\text{max}(s_{ij}[k] + \beta_{1} \alpha^*_{ij}[k]  \Delta h^{\text{rob}}_{ij}[k], 0), 1),
\end{equation}
where the difference in progress is given by 
    \begin{equation}
    \label{eq:deltah-rob}
     \Delta h^{\text{rob}}_{ij}[k] = h_{ij}(x_i^{\text{act}}[k]) - \min h_{ij}(x_i^{\text{sim}}[k]),   
    \end{equation}
\vspace{-2mm}
such that
\begin{align}
    \label{hstar-robust}
     \min h_{ij}(x_i^{\text{sim}}[k]) &= h_{ij}(x_i) \\
                 & \;\;\; + dt \left[L_f h_{ij}(x_i) + L_g h_{ij}(x_i)u_i\right] \\
                 & \;\;\; + dt \min \nabla h_{ij}(x_i)^{\top} {\Psi(x_i)},
\end{align}
and $x_i$ and $u_i$ denote $x_i^{\text{act}}[k-1]$ and $u_i[k-1]$ respectively. Equation~\eqref{hstar-robust} is an approximation of the worst case trajectory using Euler integration. The latter can be shown through the following derivation
\begin{align}
    \min h(x(t+dt)) &\approx  \min \left[ h(x(t)) + dt \frac{d h}{d t}(x(t)) \right] \\
                 &= \min \left[h(x(t)) + dt \nabla h(x(t))^{\top}\dot{x}(t) \right] \\
                 &= h(x(t)) + dt \\
                 &\;\;\;\; (L_f h(x(t)) + L_g h(x(t))u(t))  \\
                 &\;\;\; + dt \min [  \nabla h(x(t))^{\top}{\Psi(x(t))} ].
\end{align} 
Here, as reflected above, we again account for a continuum of trajectories while benefiting from the negligible computational complexity incurred by taking the minimum over a convex hull of a set of points since the minimum must occur at one of the corner points. 

As discussed in the previous subsection, since the magnitude of the learned disturbance is thresholded, large disturbances which the robots cannot overcome would be accounted for using the proposed update law. Essentially, this process allows us to differentiate between disturbances that the robots can and cannot overcome. 

\begin{figure*}[tb]
\centering
\begin{minipage}{0.33\textwidth}
\includegraphics[width=\textwidth]{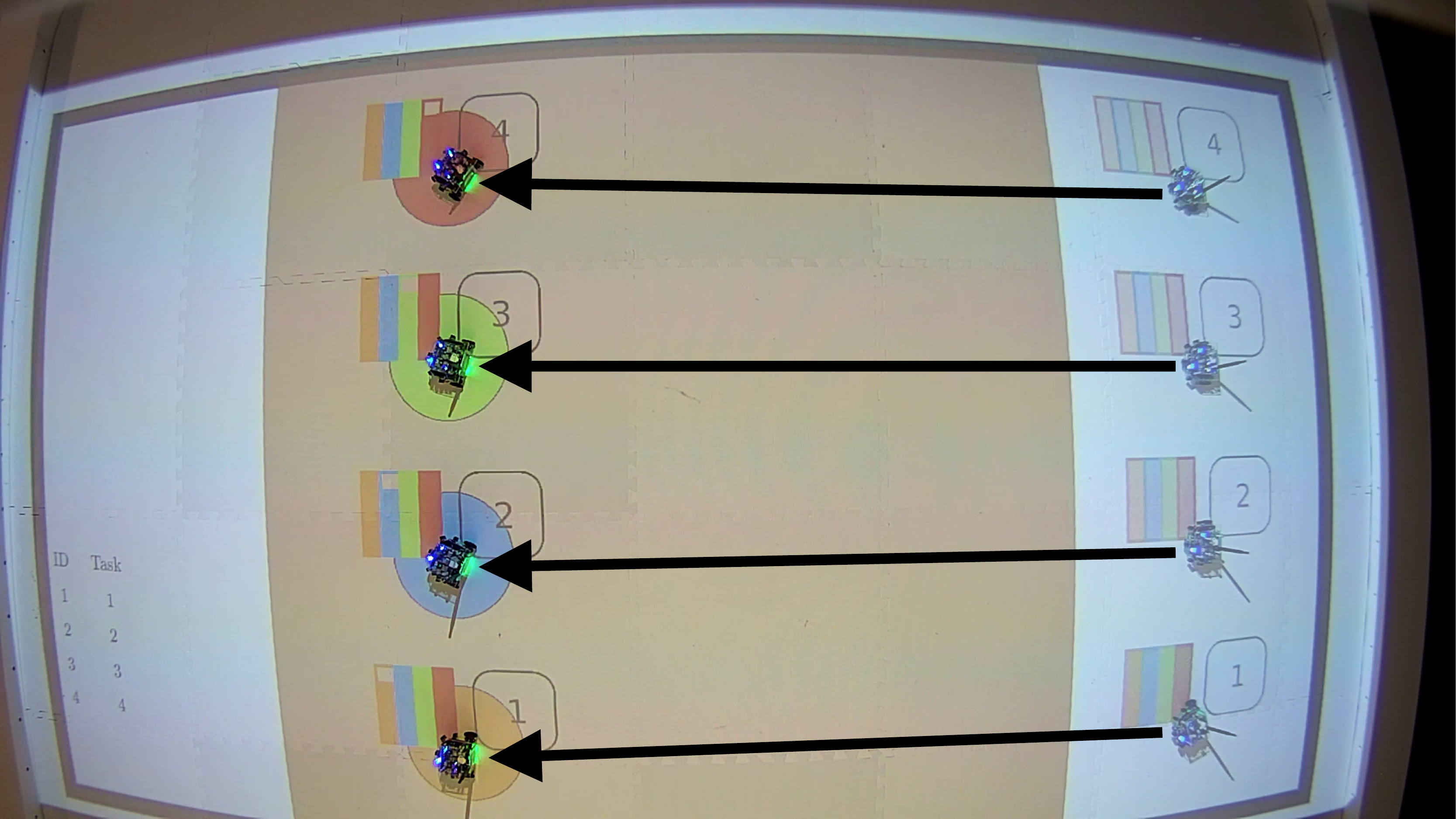}
\end{minipage}~%
\begin{minipage}{0.33\textwidth}
\includegraphics[width=\textwidth]{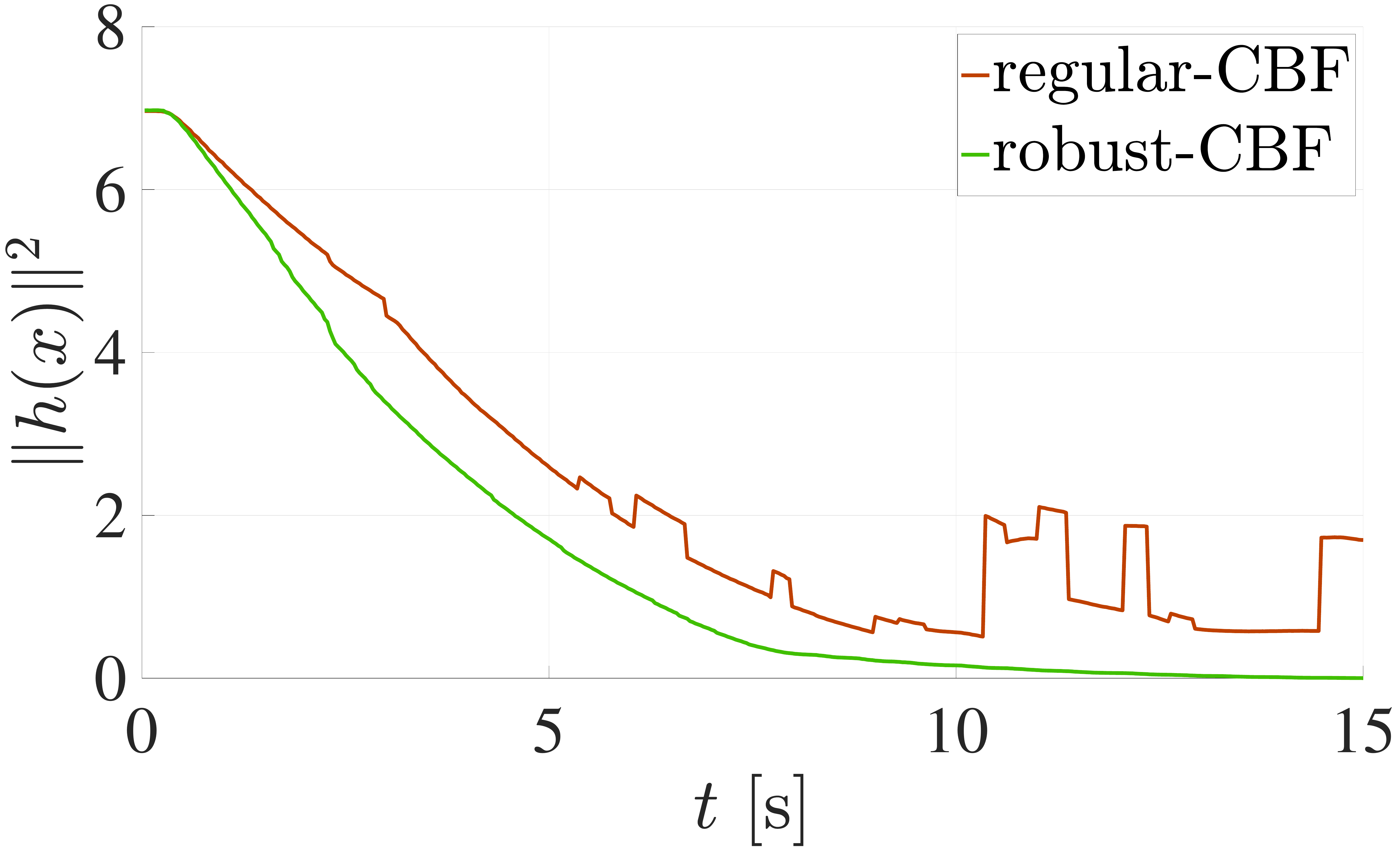}
\end{minipage}~%
\begin{minipage}{0.33\textwidth}
\includegraphics[width=\textwidth]{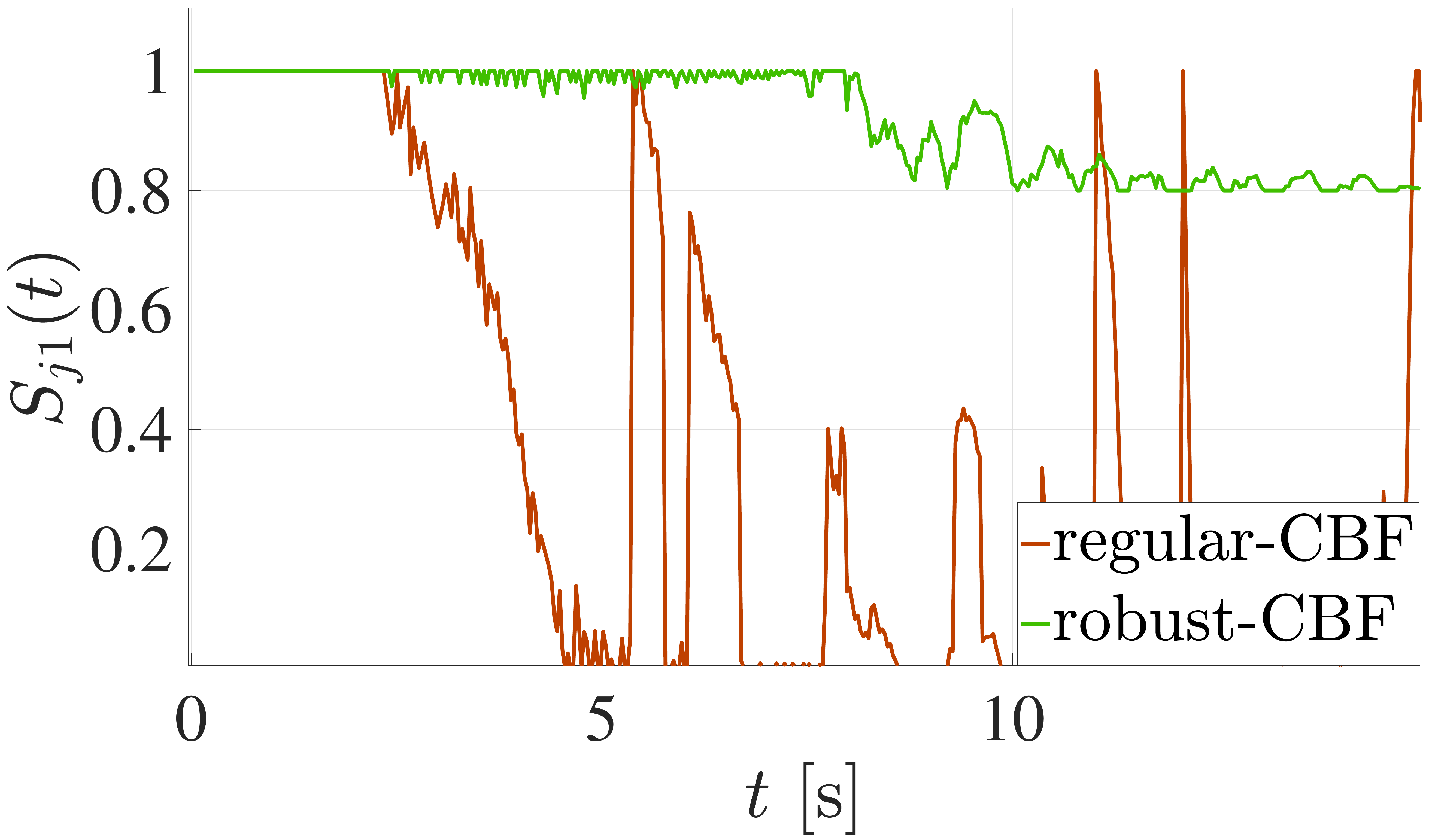}
\end{minipage}~%
\caption{The setup of experiment $1$ is shown on the left-hand side of the figure. A team of $4$ ground robots executes $4$ go-to-goal tasks (circles). The black arrows indicate the trajectories of the robots. The brown section of the arena induces a simulated disturbance on the robots by affecting the control inputs unknowingly to the framework. The sum of squares of the energy functions of the robots with respect to their assigned tasks is plotted over time (middle) for both regular and robust CBFs. Moreover, the specializations of robot $1$ towards its assigned task is plotted over time for each framework (right). As shown, the proposed framework significantly outperforms the baseline in both task execution and adapting the specialization. A video of the experiments is available at \texttt{\href{https://youtu.be/imeC-Eri\_nM}{https://youtu.be/imeC-Eri\_nM}}.}
\label{fig:exp1}
\end{figure*}

\begin{figure*}[tb]
\centering
\begin{minipage}{0.33\textwidth}
\includegraphics[width=\textwidth]{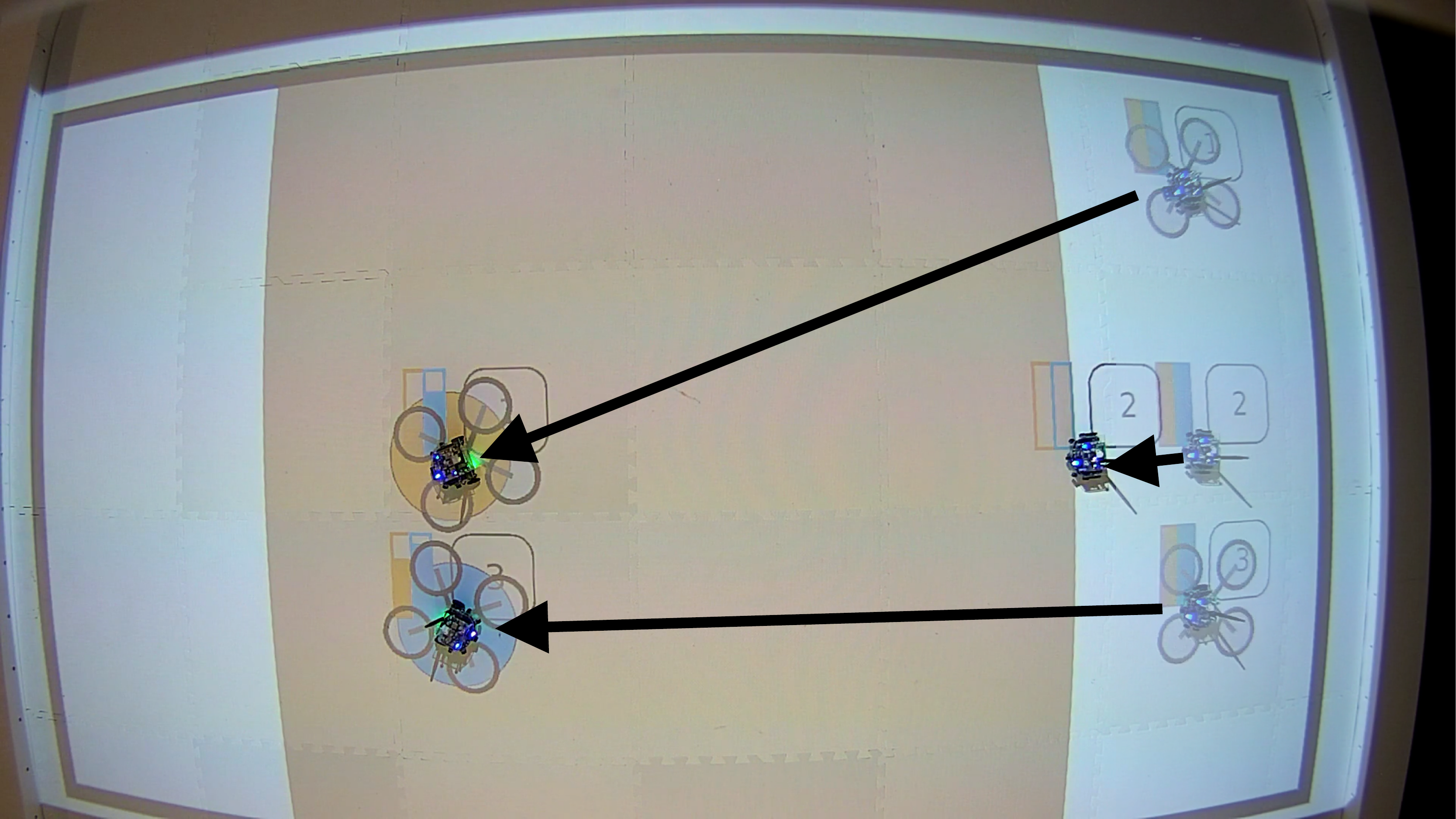}
\end{minipage}~%
\begin{minipage}{0.33\textwidth}
\includegraphics[width=\textwidth]{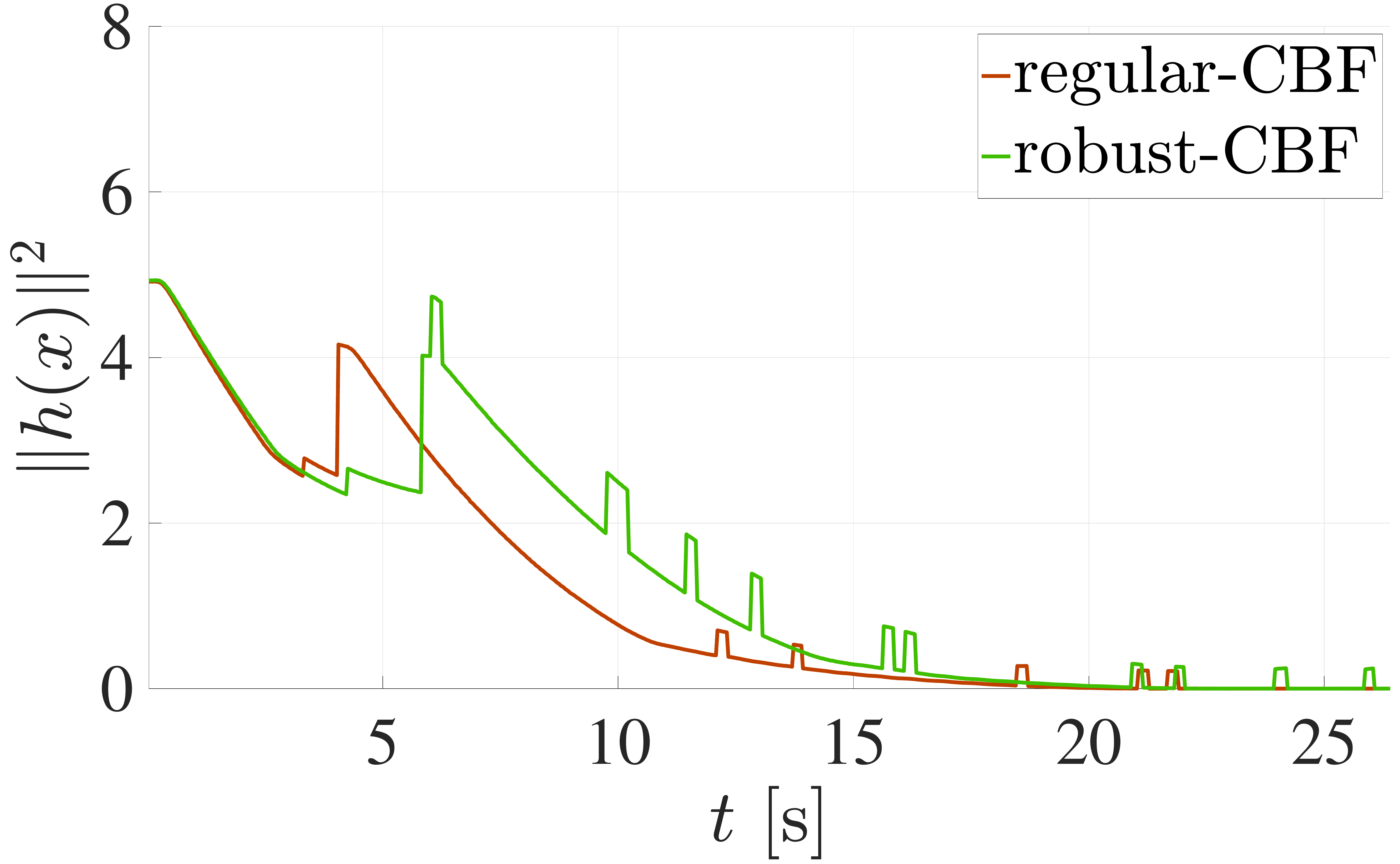}
\end{minipage}~%
\begin{minipage}{0.33\textwidth}
\includegraphics[width=\textwidth]{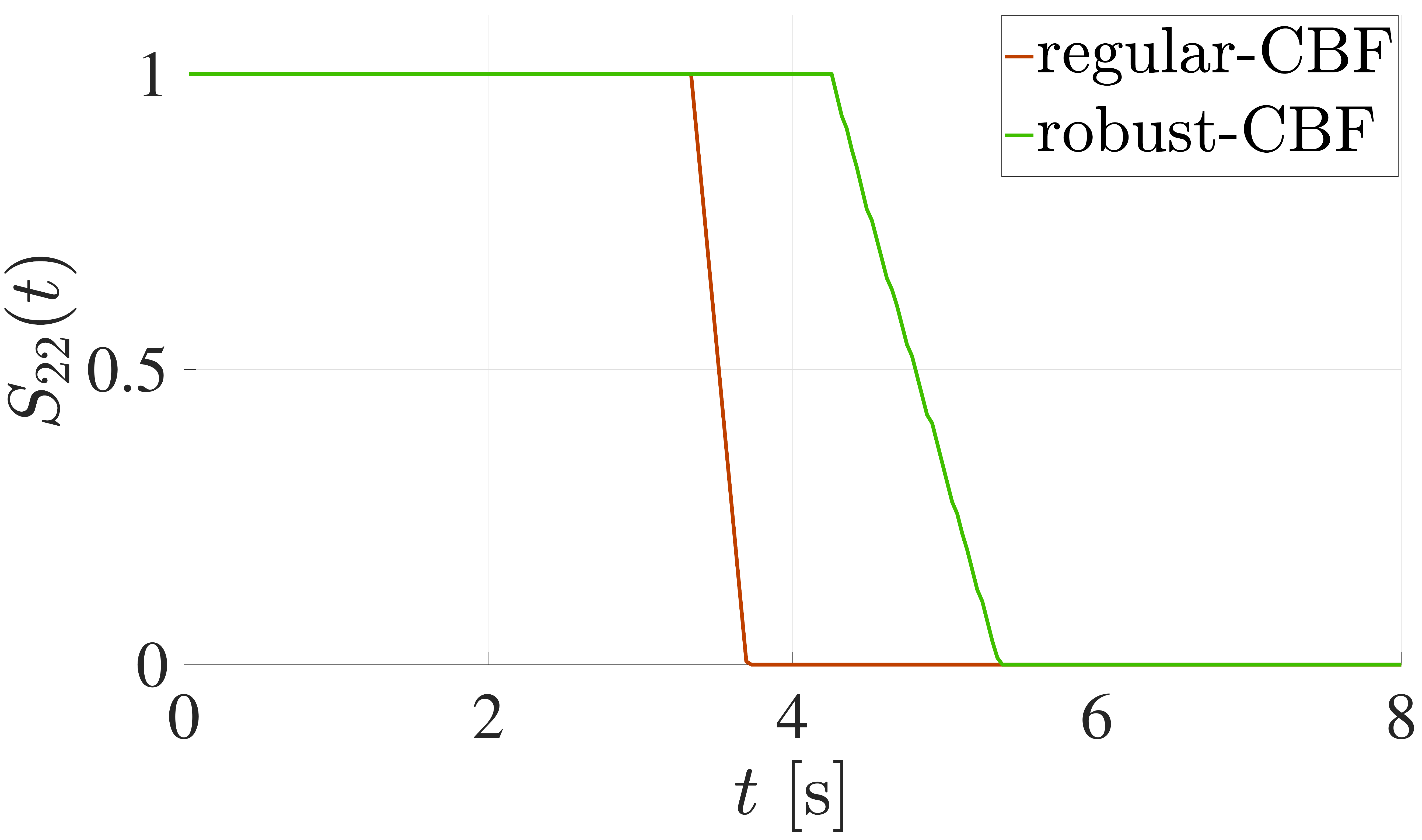}
\end{minipage}~%
 \caption{The setup of experiment $2$ is shown on the left-hand side of the figure. A team of $3$ robots, $1$ ground robot and $2$ simulated quadcopters, is to execute $2$ go-to-goal tasks (circles). Similarly to the first experiment, a simulated disturbance, which only affects the ground robot is induced in the brown region. However, in this experiment, the disturbance is very strong to the point where the ground robot cannot overcome it and reach its task. The sum of squares of the energy functions of the robots with respect to their assigned tasks is plotted over time (middle) for both regular and robust CBFs, as well as the specialization of the ground robot towards its first assigned task (right). As expected, the proposed framework is slower to adapt in this case.}
\label{fig:exp2}
\end{figure*}

Figure~\ref{fig:high-level-diagram} and Algorithm~\ref{alg:1} illustrate the proposed data-driven adaptive task allocation and execution framework. As we empirically showcase in the next section, integrating the learned disturbance models enables the framework to be robust to environmental disturbances during task execution and to update the specializations in an environment-resilient fashion.

\begin{algorithm}
	\caption{Data-Driven Adaptive MRTA}
	\label{alg:1}
 	\begin{algorithmic}[1]
	\While{true}
		\State Get robot states $x_i,\forall i$
		\State Compute inputs $u_i,\forall i$ \Comment \eqref{eq:allocationalgorithm} with \eqref{eq:miqp-b-robust} in lieu of \eqref{eq:miqp:b}
		\State Execute inputs $u_i, \forall i$  
		\ForAll{$i\in\{1,\ldots,n_r\}$}
			\ForAll{$j\in\{1,\ldots,n_t\}$}
				\State Evaluate $\Delta h^{\text{rob}}_{ij}[k]$ \Comment \eqref{eq:deltah-rob}
				\State Evaluate $s_{ij}[k+1]$ \Comment \eqref{eq:spUpdate-robust}
			\EndFor
		\EndFor
	\EndWhile
 	\end{algorithmic}
\end{algorithm}

\section{Experiments}
\label{sec:experiments}

We showcase the significant improvement obtained from the proposed framework as compared to the baseline from \cite{emam2020adaptive} in two experiments on the Robotarium \cite{pickem2017}. In each experiment, we consider a team of differential-drive robots each with a state $x_{i} \coloneqq \begin{bmatrix}x_{i, 1} ~ x_{i, 2} ~ \theta_{i}\end{bmatrix}^{\top}$, output $p_{i} \coloneqq \begin{bmatrix}p_{i, 1} ~ p_{i, 2}\end{bmatrix}^{\top}$ and input $u_{i} \coloneqq \begin{bmatrix}v_i ~ \omega_i\end{bmatrix}^{\top}$. Each robot nominally obeys the following dynamics
\begin{equation}
    \dot{x}_{i} =
    \begin{bmatrix}
        \cos \theta_{i} & 0 \\ 
        \sin \theta_{i} & 0 \\ 
        0 & 1
    \end{bmatrix}
    u_{i}, \;\;
    \dot{p}_{i} =
    \begin{bmatrix}
        \cos \theta_{i} & -\sin \theta_{i} \\
        \sin \theta_{i} & \cos \theta_{i} 
    \end{bmatrix} 
    \begin{bmatrix}
        1 & 0 \\ 
        0 & l_{p}
    \end{bmatrix}
    u_{i},
\end{equation}
where $l_{p}$ is the look-ahead projection distance. The disturbances are then simulated by altering the control input as explained in each experiment. We note that the CBFs are formulated with respect to the output. Moreover, for the RCBFs, the GP models are learned \textit{a priori} and the magnitude of the disturbance is limited by $d_{\text{max}} = 0.10$. We note that learning the disturbance online is left for future work.

The setup of the first experiment is shown on the left-hand side of Figure~\ref{fig:exp1}. Four robots are to execute four go-to-goal tasks depicted by the circles. A simulated disturbance is induced in the brown section of the arena, where the disturbed control input applied is given by
$u_{d} = u + 0.02cos(\theta) \begin{bmatrix}1 ~ 0 \end{bmatrix}^{\top}$. This disturbance aims to simulate a sloped surface, which the robots must climb to reach the tasks. The experiment is then ran using both the proposed and baseline frameworks. Shown in the middle of Figure~\ref{fig:exp1} is the sum of squares of the energy functions between the robots and their assigned tasks over time. As the robots progress towards accomplishing the tasks, this sum should approach zero, which is the case for the proposed framework but not the baseline framework. Moreover, to the right-hand side of Figure~\ref{fig:exp1} showcases the specialization of robot $1$ towards its assigned task at each time step for each framework. As opposed to the proposed framework, not accounting for the disturbance in the update law of the baseline framework causes the specialization to reach zero for all tasks, although the robot is indeed executing the task. This is in turn causes a cycle, where the baseline framework assigns the robot another task, for which its specialization again decreases to zero as demonstrated by the periodicity of the plot. 

The second experiment aims to demonstrate that if a disturbance is large enough the robots cannot overcome it, the proposed framework will indeed adapt and generate a different allocation of tasks. The setup of the experiment is shown on the left-hand side of Figure~\ref{fig:exp2}. A team composed of $1$ ground robots and $2$ simulated quadcopters is to execute $2$ go-to-goal tasks. Similarly to the first experiment, a disturbance is induced in the brown area of the arena resulting in the disturbed control input $u_{d} = u + 0.2cos(\theta) \begin{bmatrix}1 ~ 0 \end{bmatrix}^{\top}$. This disturbance only affects the ground robot. Note that, as opposed to the first experiment, the ground robot cannot overcome this disturbance due to its control limits.

Shown in the middle of Figure~\ref{fig:exp1} is the sum of squares of the energy functions between the robots and their assigned tasks over time. Initially, the ground robot and the bottom quadcopter are assigned the orange and blue tasks respectively. However, when the ground robot faces the disturbance (slope), it cannot overcome it and its specialization decreases towards its task (as shown in the right-hand plot of Figure~\ref{fig:exp2}).  This causes a new allocation of tasks where the top quadcopter is assigned the orange task which is depicted by the spike in the energy as shown in the middle plot of Figure~\ref{fig:exp2}. Comparing the performance of both frameworks, as shown in the energy plot of Figure~\ref{fig:exp2}, the proposed framework is slightly delayed compared to the baseline. This is because the proposed framework accounts for a portion of the disturbance, limited by the magnitude $d_{\text{max}}$. As such, the difference between modelled and actual progresses in the proposed framework is smaller in the baseline resulting a slower change in the specializations. This phenomena represents a trade-off between robustness and adaptivity, where the robots either attempt to accomplish the task to the best of their ability, or re-configure so that a more suited robot can execute the task.

\section{Conclusion} 
\label{sec:conclusion}
In this paper, we introduce a task allocation and execution framework for heterogeneous robot teams that explicitly accounts for environmental disturbances by augmenting the pre-defined dynamics models with a learned component. As shown empirically, the integration of the learned dynamics using robust control barrier functions and a novel adaptive specialization law renders the framework robust to disturbances regarding task execution and permits the framework to differentiate between disturbances that the robots can and cannot overcome.

\bibliographystyle{unsrt}
\bibliography{ref.bib}

\end{document}